\documentclass{tlp}

\usepackage[hyphens]{url}
\usepackage{hyperref}
\usepackage{amssymb,amsfonts,amsmath}
\usepackage{algorithm}
\usepackage{algorithmicx,algpseudocode}
\usepackage{listings}
\usepackage{thmtools}
\usepackage{thm-restate}
\usepackage{cleveref}
\usepackage{tikz}
\usetikzlibrary{arrows}
\tikzstyle{arg}=[draw, thick, circle]
\usepackage{etoolbox}
\appto\UrlBreaks{\do\a\do\b\do\c\do\d\do\e\do\f\do\g\do\h\do\i\do\j
\do\k\do\l\do\m\do\n\do\o\do\p\do\q\do\r\do\s\do\t\do\u\do\v\do\w
\do\x\do\y\do\z}

\newtheorem{definition}{Definition} 
\newtheorem{example}{Example} 

\declaretheorem{thm}
\declaretheorem[name=Proposition,sibling=thm]{prop}

\crefformat{footnote}{#2\footnotemark[#1]#3}

\newcommand{\SigmaP}[1]{{\rm \Sigma}_{#1}^{P}}
\newcommand{\PiP}[1]{{\rm \Pi}_{#1}^{P}}

\newcommand{\U}{{\ensuremath{\cal U}}}

\newcommand{\cf}{\mathit{cf}}
\newcommand{\stable}{{\mathit{stb}}}
\newcommand{\adm}{\mathit{adm}}
\newcommand{\pref}{\mathit{prf}}
\newcommand{\stage}{\mathit{stage}}
\newcommand{\semi}{\mathit{sem}}

\newcommand{\AF}{F}

\newcommand{\NP}{\mbox{\rm NP}}

\newcommand{\naf}{{\it not}\,}
\newcommand{\la}{\leftarrow}
\newcommand{\aspmodule}[1]{\pi_{\mathit{#1}}}
\newcommand{\answersets}{\mathcal{AS}}

\newcommand{\aspformat}[1]{{\bf #1}}
\newcommand{\argument}{\aspformat{arg}}
\newcommand{\att}{\aspformat{att}}
\newcommand{\inA}{\aspformat{in}}
\newcommand{\outA}{\aspformat{out}}

\newcommand{\aspinf}{\aspformat{inf}}
\newcommand{\aspsucc}{\aspformat{succ}}
\newcommand{\aspsup}{\aspformat{sup}}
\newcommand{\undefended}{\aspformat{undefended}}
\newcommand{\defeated}{\aspformat{defeated}}

\newcommand{\spoil}{\aspformat{spoil}}
\newcommand{\ecl}{\aspformat{witness}}
\newcommand{\nontrivial}{\aspformat{nontrivial}}
\newcommand{\ok}{\aspformat{unstable}}
\newcommand{\rge}{\aspformat{range}}
\newcommand{\nrge}{\aspformat{out\_of\_range}}

\newcommand{\inn}{\aspformat{inN}}
\newcommand{\outn}{\aspformat{outN}}

\newcommand{\BU}{{\ensuremath{B_{\cal U}}}}
\newcommand{\head}[1]{H(#1)}
\newcommand{\body}[1]{B(#1)}
\newcommand{\bodyp}[1]{B^{+}(#1)}
\newcommand{\bodyn}[1]{B^{-}(#1)}

\newcommand{\UP}{\ensuremath{U_{\pi}}}
\newcommand{\GP}{\ensuremath{Gr(\pi)}}
\newcommand{\Gr}{\ensuremath{Gr}}
\newcommand{\AS}{\mathcal{AS}}

\newcommand{\enf}{\hat{\AF}}
\newcommand{\encf}{\aspmodule{\cf}}

\newcommand{\encdefense}{\aspmodule{\mathit{def}}}
\newcommand{\enadm}{\aspmodule{\adm}}
\newcommand{\enadmp}{\aspmodule{\adm}^{+}}

\newcommand{\enpreftwo}{\aspmodule{\pref ^2}}

\newcommand{\enstagetwo}{\aspmodule{\stage ^2}}

\newcommand{\ensemitwo}{\aspmodule{\semi ^2}}

\newcommand{\ensatpreftwo}{\aspmodule{sat\pref ^2}}

\newcommand{\ensatsemitwo}{\aspmodule{sat\semi ^2}}
\newcommand{\enrange}{\aspmodule{range}}

\lstdefinelanguage{asp}
{morekeywords={in,out,arg,d,att,cancel,succeed,in_range,rs,attacked,lt,nsucc,succ,inf,ninf,nsup,sup,rsinit,todef,unattacked_upto,unattacked,remove,defeated,nontrivial,spoil,undefended,ecl,rge,range,lrge,larger_range,nrge,out_of_range,ok,unstable,witness,eq_upto,eq,outN,inN}, 
literate={:-}{{$\la$}}1 {not}{{$\naf$}}2 {\\el}{{$\in$}}1 {\\cup}{{$\cup$}}1,
morecomment=[l]{\%},
}

\makeatletter
\let\@ORGmakecaption\@makecaption
\long\def\@makecaption#1#2{\@ORGmakecaption{#1}{#2}\vskip\belowcaptionskip\relax}
\makeatother

\begin{document}
\bibliographystyle{acmtrans}

\long\def\comment#1{}


\title{%
Improved Answer-Set Programming Encodings for Abstract Argumentation
}

\author[S. A. Gaggl, N. Manthey, A. Ronca, J. P. Wallner, and S. Woltran]
{Sarah A. Gaggl, Norbert Manthey \\
Technische Universit\"at Dresden, Germany
\and
Alessandro Ronca\\
La Sapienza, University of Rome
\and 
Johannes P. Wallner\\
HIIT, Department of Computer Science, University of Helsinki, Finland
\and
Stefan Woltran\\
Vienna University of Technology, Austria
}


\maketitle

\label{firstpage}

\begin{abstract}
The design of efficient solutions for abstract argumentation problems is 
a crucial step towards advanced argumentation systems. One of the most
prominent approaches in the literature is to use 
Answer-Set Programming (ASP) for this endeavor.
In this paper, 
we present new encodings for three prominent argumentation
semantics using the concept of conditional literals in disjunctions as provided by the 
ASP-system
clingo. Our new encodings are not only more succinct than previous
versions, but also outperform them on standard benchmarks.
\end{abstract}

\begin{keywords}
Answer-Set Programming, Abstract Argumentation, Implementation, ASPARTIX
\end{keywords}

\section{Introduction}

Abstract Argumentation \cite{Dung1995,RahwanS2009}
is at the heart of many advanced argumentation systems~\cite{BesnardH2008,CaminadaA2007} 
and is concerned with finding jointly acceptable arguments 
by taking only their inter-relationships into account.
Efficient solvers for abstract argumentation are thus an important development, 
a fact that is also witnessed by a new competition
which takes place in 2015 for the first time~\cite{CeruttiOSTV14}\footnote{See \url{http://argumentationcompetition.org} for further information.}.

To date, several approaches 
for implementing abstract argumentation
exist, many of them following the so-called reduction-based (see \cite{CharwatDGWW15})
paradigm: 
hereby, existing efficient software which has
originally been developed for other purposes is used. 
Prominent examples for this approach 
are 
(i) the CSP-based system ConArg~\cite{BistarelliS11a},
(ii) SAT-based approaches (e.g.\ \cite{CeruttiGV14a,DvorakJWW2014}) 
and 
(iii)
systems which rely on Answer-Set Programming (ASP); see 
\cite{Toni11} for a comprehensive survey.
In fact, ASP \cite{BrewkaET11} is 
particularly well-suited since 
ASP systems by default enumerate all solutions of a given program, thus
enabling the enumeration of extensions of an abstract argumentation framework
in an easy manner. Moreover, 
disjunctive ASP is capable of expressing problems being even complete for the 2nd level of the polynomial hierarchy.
In fact, several 
semantics for abstract argumentation like preferred, semi-stable~\cite{CaminadaCD12}, or stage~\cite{Verheij96} are of this high complexity \cite{DunneB02,DvorakW10}.

One particular candidate for an ASP reduction-based system is 
ASPARTIX \cite{EglyGW2010,DvorakGWW2013}. Here, a fixed program for each semantics is provided and the argumentation framework under consideration is just
added as an input-database. The program together with the input-database is then handed over to an ASP system of choice in order to calculate the extensions. This makes the ASPARTIX approach easy
to adapt and an appealing rapid-prototyping method. The proposed
encodings in ASPARTIX for the high-complexity semantics mentioned above come, however, with a certain caveat. This stems from the fact that encodings 
for such complex programs have to follow a certain saturation pattern, 
where restricted use of cyclic negation has to be taken care of
(we refer to \cite{EglyGW2010} for a detailed discussion). The original
encodings followed the definition of the semantics quite closely and
thus resulted in quite complex and tricky loop-techniques  which 
are a known feature for ASP experts, but hard to follow for ASP laymen.
Moreover, experiments in other domains indicated that such loops
also potentially lead to performance bottlenecks.

In this work, we thus 
aim for new and simpler encodings for the three semantics of preferred, semi-stable, and stage extensions. To this end, we provide some alternative characterizations for these semantics and design our new encodings along these characterizations in such a way that costly loops are avoided. 
Instead we make use of the ASP language feature of conditional literals in disjunction~\cite{syrjanen09a,PotasscoUserGuide}. 
Moreover,
we perform exhaustive 
experimental evaluation 
against the original ASPARTIX-encodings, the ConArg system, and another
ASP-variant \cite{DvorakGWW2013}
which makes use of the ASP front-end \emph{metasp}~\cite{GebserKS11}, where the required
maximization is handled via meta-programming.
Our results show that the new ASP encodings not only outperform the 
previous variants, but also makes  ASPARTIX more powerful than ConArg. 

The novel encodings together with the benchmark instances are available under
\url{http://dbai.tuwien.ac.at/research/project/argumentation/%
systempage/#conditional}. 
\paragraph{Acknowledgements}
This work has been funded by the Austrian Science Fund (FWF) through projects Y698 and I1102, by the German Research Foundation (DFG) through project HO 1294/11-1, and by Academy of Finland through grants 251170 COIN and 284591.  
%
%
%
%
\section{Background}

%
%

\subsection{Abstract Argumentation}

First, we recall the main formal ingredients for argumentation frameworks~\cite{Dung1995,Baroni2011} 
and survey relevant complexity results (see also~\cite{DunneW09}).

\begin{definition}\label{def:af}
An {\em argumentation framework (AF)} is a pair $\AF=(A,R)$ where $A$ 
is a set of arguments and $R \subseteq A \times A$ is the attack relation.
The pair 
$(a,b) \in R$ means that $a$ attacks 
$b$. 
An argument $a \in A$ is {\em defended} 
by a set $S \subseteq A$ 
if, for each $b \in A$ such that $(b,a) \in R$, 
there exists a $c \in S$ such that $(c,b) \in R$.
We 
define 
the \emph{range of $S$} (w.r.t.\ $R$) as $S^{+}_{R}=S\cup \{ x\mid \exists y\in S \text{ such that } (y,x) \in R\}$. 
 \end{definition}

%

Semantics for argumentation frameworks 
are given via a
function $\sigma$ which assigns 
to each AF $\AF=(A,R)$ 
a set
$\sigma(\AF)\subseteq 2^A$ of extensions. 
We shall consider here for 
$\sigma$ the functions
$\stable$, $\adm$, $\pref$,  
$\stage$, and $\semi$ which 
stand for 
stable, admissible, preferred, 
stage, and  semi-stable semantics respectively.
%


\begin{definition}\label{def:semantics}
Let $\AF=(A,R)$ be an AF.  A set $S\subseteq A$ is 
{\em conflict-free (in $\AF$)}, 
  if there are no 
$a, b \in S$, such that $(a,b) \in R$.
$\cf(\AF)$ denotes the collection of conflict-free sets of $\AF$.
For a conflict-free set $S \in \cf(\AF)$, it holds that
\begin{itemize}
\item  
$S\in\stable(\AF)$, 
if  $S^{+}_{R} = A$;
\item 
$S\in\adm(\AF)$, 
if  each $s\in S$ is defended by $S$; 
\item 
$S\in\pref(\AF)$, 
if $S\in\adm(\AF)$ and 
there is no $T\in\adm(\AF)$ with $T\supset S$;
\item 
$S\in\semi(\AF)$, if
$S\in\adm(\AF)$
and there is no 
$T\in\adm(\AF)$ with
$T^{+}_{R}\supset S^{+}_{R}$;
\item 
$S\in\stage(\AF)$,
if there is no $T\in\cf(\AF)$ in $\AF$, 
such that
$T^{+}_{R}\supset S^{+}_{R}$.
\end{itemize}
\end{definition}
%
%
\begin{example}\label{example:AF}
 Consider the AF $F=(A,R)$ with
    $A=\{a,b,c,d,e,f\}$ and
    $R=\{(a,b)$, $(b,d)$, $(c,b)$, $(c,d)$, $(c,e)$, $(d,c)$, $(d,e)$, $(e,f)\}$,
and the graph representation of $F$: 

\begin{center}
\begin{tikzpicture}[scale=1.4,>=stealth']
		\path 	node[arg](a){$a$}
			++(1,0) node[arg,inner sep=3](b){$b$}
			++(1,.4) node[arg](c){$c$}
			++(0,-0.8) node[arg,inner sep=2.8](d){$d$}
			++(1,0.4) node[arg](e){$e$}
			++(1,0) node[arg,inner sep=1.8](f){$f$};
			;
		\path [left,->, thick]
			(a) edge (b)
			(c) edge (b)
			(b) edge (d)
			(d) edge (e)
			(c) edge (e)
			(e) edge (f)
			;
		\path [bend left, left, above,->, thick]
			(c) edge (d)
			(d) edge (c);
\end{tikzpicture}
\end{center}
%
%
We have 
$\stable(F)=\stage(F)=\semi(F)=\{\{a,d,f\},\{a,c,f\}\}$. 
The admissible sets of $F$ are
$\emptyset$, $\{a\}$, $\{c\}$, $\{a,c\}$, $\{a,d\}$, $\{c,f\}$, $\{a,c,f\}$, $\{a,d,f\}$, 
and $\pref(F)=\{\{a,c,f\}$,$\{a,d,f\}\}$.
%
\end{example}

We recall that each AF $F$ possesses at least one preferred, semi-stable, and stage extension, while $\stable(F)$ might be empty. However, it is well known 
that $\stable(F)\neq\emptyset$ implies 
$\stable(F)=\stage(F)=\semi(F)$ as also seen in the above example.

Next, we provide some alternative characterisations for the
semantics of our interest. They will serve as the basis of our encodings.

The alternative characterisation for preferred extensions relies
on the following idea. An admissible set $S$ is preferred, if each other
admissible set $E$ (which is not a subset of $S$) is in conflict with $S$.

\begin{prop}
\label{prop:2}
Let $F=(A,R)$ be an AF and $S\subseteq A$
be admissible in $F$. Then, 
$S\in\pref(F)$ if and only if,
for each $E\in\adm(F)$ such that $E\not\subseteq S$,
$E\cup S\notin\cf(F)$.
\end{prop}
\begin{proof}
Let $S\in\adm(F)$ and 
assume there exists an admissible (in $F$) set  $E\not\subseteq S$, such that $E\cup S\in\cf(F)$.  It is well known (see, e.g. \cite{DunneDLW14}, Lemma 1) that if two sets $E_1,E_2$ defend themselves in an AF $F$, then also $E_1\cup E_2$ defends itself in $F$.
It follows that $E\cup S\in\adm(F)$ and by assumption $S\subset E\cup S$. Thus, $S\notin\pref(F)$.
For the other direction, let $S\in\adm(F)$ but $S\notin\pref(F)$.
Hence, there exists an $S'\supset S$ such that $S'\in\adm(F)$.
Clearly, $S'\not\subseteq S$ but $S'=(S\cup S')\in\cf(F)$. 
\end{proof}

We turn to semi-stable and stage semantics. 
In order to verify
whether a candidate extension $S$ is a stage (resp.\ semi-stable) extension 
of an AF $F$, 
we
check whether for any set 
$S'$ such that $S'\supset S^{+}_R$ 
there is no conflict-free (resp. admissible) set $E$
such that $S'\subseteq E^{+}_R$.  We also show that is sufficient to check this for minimal such 
sets $S'$.
Observe that the above check 
is trivially true if $S$ is already stable, mirroring the observation
that 
$\stable(F)=\stage(F)=\semi(F)$ whenever $\stable(F)\neq\emptyset$.

\begin{definition}
Let $F=(A,R)$ be an AF and $S\subseteq A$. 
A \emph{cover} of $S$ in $F$ is any 
$E\subseteq A$ such that 
$S
\subseteq 
E^{+}_R$. 
The set of covers of $S$ in $F$ is denoted by $\Gamma_F(S)$.
\end{definition}

\begin{prop}
\label{prop:1}
Let $F=(A,R)$ be an AF and $S\in\cf(F)$ (resp.\ $S\in\adm(F))$. The
following propositions are equivalent:
(1) $S$ is a stage (resp.\ semi-stable) extension of $F$;
(2) for each $a\in A\setminus S^{+}_R$, there is no $E\in\Gamma_F(S^{+}_R\cup \{a\})$ 
such that $E\in\cf(F)$ (resp.\ $E\in\adm(F)$;
(3) for each $S'$ 
with $S^{+}_R\subset S'\subseteq A$, 
there is no 
$E\in\Gamma_F(S')$,
such that $E\in\cf(F)$ (resp.\ $E\in\adm(F)$).
\end{prop}
\begin{proof}
We give the proof for stage extensions. The result for semi-stable  proceeds analogously.
(1)$\Rightarrow$(3): Suppose there is an $S'$ 
with $S^{+}_R\subset S'\subseteq A$, such that some
$E\in\Gamma_F(S')$ is
conflict-free in $F$. By definition, 
$
S^{+}_R
\subset 
S' 
\subseteq 
E^{+}_R
$. 
Hence, $S\notin\stage(F)$.
(2)$\Rightarrow$(1): Suppose $S\notin\stage(F)$. Thus there exists $T\in\cf(F)$ with $S^{+}_R\subset T^{+}_R$. Let $a\in T^{+}_R \setminus S^{+}_R$. It follows 
that $T\in\Gamma_F(S\cup\{a\})$. 
(3)$\Rightarrow$(2) is clear.
\end{proof}

Finally,
we turn to the complexity of reasoning in AFs for two major decision problems. 
%
For a given AF $F=(A,R)$ and an argument $a\in A$, credulous reasoning under $\sigma$ denotes the problem of deciding whether there exists an $E \in \sigma(F)$ s.t.\ $a \in E$. 
Skeptical Acceptance under $\sigma$ is the problem of deciding whether for all $E \in \sigma(F)$ it holds that $a \in E$. 
%
%
Credulous reasoning for preferred semantics is $\NP$-complete, while credulous reasoning for semi-stable and stage semantics is $\SigmaP{2}$-complete. For preferred, semi-stable, and stage semantics skeptical reasoning is $\PiP{2}$-complete~\cite{Dung1995,DimopoulosT1996,DunneB02,CaminadaD2008,DvorakW10}.



\subsection{Answer-Set Programming}
We give an overview of the syntax and semantics of disjunctive logic programs under the answer-sets semantics~\cite{GelfondL91}. 

We fix a countable set $\U$ of {\em (domain) elements}, also called \emph{constants};
and suppose a total order $<$ over the domain elements.
An {\em atom} is an expression
$p(t_{1},\ldots,t_{n})$, where $p$ is a {\em predicate} of arity $n\geq 0$
and each $t_{i}$ is either a variable or an element from $\U$.
An atom is \emph{ground} if it is free of variables.
$\BU$ denotes the set of all ground atoms over $\U$.
%
A \emph{(disjunctive) rule} $r$ is of the form
\begin{equation}
a_1\ \vert\ \cdots\ \vert\ a_n\ \la
        b_1,\ldots, b_k,\
        \naf b_{k+1},\ldots,\ \naf b_m
\end{equation}
with $n\geq 0,$ $m\geq k\geq 0$, $n+m > 0$, where
$a_1,\ldots ,a_n,b_1,\ldots ,b_m$ are
literals, and ``$\naf$'' stands for {\em default negation}.
The \emph{head} of $r$ is the set
$\head{r}$ = $\{a_1, \ldots, a_n\}$ and
the \emph{body} of $r$ is
$\body{r}=
\{b_1,\ldots, b_k,$ $\naf b_{k+1},\ldots,$ $\naf  b_m\}$.
Furthermore, $\bodyp{r}$ = $\{b_{1},\ldots, b_{k}\}$ and
$\bodyn{r}$ = $\{b_{k+1},\ldots, b_m\}$.
A rule $r$ is \emph{normal} if $n \leq 1$ and a 
\emph{constraint} if $n=0$. 
A rule $r$ is \emph{safe} if each variable in $r$ occurs in $\bodyp{r}$.
A rule $r$ is \emph{ground} if no variable occurs in $r$.
A \emph{fact} is a ground rule without disjunction and empty body. 
An \emph{(input) database} is a set of facts.
A program is a finite set of disjunctive rules. 
For a program $\pi$ and an input database $D$, we often write $\pi{(D)}$ instead of $D\cup\pi$.
If each rule in a program is 
normal (resp.\ ground), we call the program normal (resp.\ ground).

For any program $\pi$, let \UP{}
be the set of all constants appearing in $\pi$.
$\GP$ is  the set of rules
$r\sigma$
obtained by applying, to each rule 
$r\in\pi$, all possible
substitutions $\sigma$ from the variables
in $r$ to elements of $\UP{}$.
%
An \emph{interpretation} $I\subseteq \BU$ 
\emph{satisfies} 
a ground rule $r$
iff $\head{r} \cap I \neq \emptyset$ whenever
$\bodyp{r}\subseteq I$ and $\bodyn{r} \cap I = \emptyset$.
$I$ satisfies a ground program $\pi$,
if each $r\in\pi$
is satisfied by $I$.
A non-ground rule $r$ (resp., a program $\pi$)
is satisfied by an interpretation $I$ iff
$I$ satisfies all groundings of $r$ (resp., $\GP$).
%
$I \subseteq \BU$ is an \emph{answer set}
of $\pi$
iff it is a subset-minimal set
satisfying
the \emph{Gelfond-Lifschitz reduct}
$
\pi^I=\{ \head{r} \la \bodyp{r} \mid I\cap
\bodyn{r} = \emptyset, r \in \GP\}
$.
%
For a program $\pi$,
we denote the set of its 
answer sets by 
$\AS(\pi)$.

Modern ASP solvers offer additional language features. Among them we make use of the \emph{conditional literal}~\cite{syrjanen09a,PotasscoUserGuide}. In the head of a disjunctive rule literals may have conditions, e.g.\ consider the head of rule ``$\aspformat{p}(X) : \aspformat{q}(X) \la$''. Intuitively, this represents a head of disjunctions of atoms $\aspformat{p}(a)$ where also $\aspformat{q}(a)$ is true.

\subsection{ASP Encodings for AFs}
For our novel encodings we utilize basic encodings for AFs, conflict-free sets, and admissible sets from~\cite{EglyGW2010}. An AF is represented as a set of facts. 
\begin{definition}
Let $\AF =(A,R)$ be an AF. We define $\enf = \{\argument(a) \mid a \in A\} \cup \{\att(a,b) \mid ( a,b ) \in R\}$.
\end{definition}

In the following definition we first formalize the correspondence between an extension, as subset of arguments, and an answer set of an ASP encoding; then we extend it to the one between sets of extensions and answer sets respectively.
\begin{definition} \label{def:correspondence}
Let $\mathcal{S} \subseteq 2^{\U}$ be a collection of sets of domain elements and let $\mathcal{I} \subseteq 2^{\BU}$ be a collection of sets of ground atoms. 
We say that $S \in \mathcal{S}$ and $I \in \mathcal{I}$ correspond to each other, in symbols $S \cong I$,
iff $S = \{ a \mid \inA(a) \in I \}$.
We say that $\mathcal{S}$ and $\mathcal{I}$ correspond to each other, in symbols $\mathcal{S} \cong \mathcal{I}$, iff 
(i) for each $S \in \mathcal{S}$, there exists an $I \in \mathcal{I}$, such that $I \cong S$; and (ii) for each $I \in \mathcal{I}$, there exists an $S \in \mathcal{S}$, such that $S \cong I$.
\end{definition}

It will be convenient to use the following notation and result later in~\Cref{sec:encodings}.
\begin{definition}
Let $I,J \in 2^{\BU}$ be sets of ground atoms. 
We say that $I$ and $J$ are equivalent, in symbols $I \equiv J$, iff 
$\{ \inA(a) \mid \inA(a) \in I \} = \{ \inA(a) \mid \inA(a) \in J \}$. 
%
\end{definition}
\begin{restatable}{lma}{lmaequiv} \label{th:equiv}
Let $I,J \in 2^{\BU}$, and $S \in 2^{\U}$. If $I \equiv J$ and $I \cong S$, then $J \cong S$.
\end{restatable}

%
%
%
%

In~\Cref{asp:cf} we see the ASP encoding for conflict-free sets, while~\Cref{asp:def} shows defense of arguments. The encoding for admissible sets is given by $\enadm = \encf \cup \encdefense$. 
The following has been proven in~\cite[{{Proposition 3.2}}]{EglyGW2010}.
\begin{prop} \label{prop:partition}
For any AF $\AF=(A,R)$, and any $I \in \answersets(\encf(\enf))$, $\mathcal{P} = \{\{a \mid \inA(a) \in I\}, \{a \mid \outA(a) \in I\}\}$ is a partition of $A$.
\end{prop}
%
\begin{lstlisting}[caption=Module $\encf$,label=asp:cf,float=tb,belowskip=-0.8 \baselineskip]
in(X) :- arg(X), not out(X).§\label{line:cf-r1}§
out(X) :- arg(X), not in(X).§\label{line:cf-r2}§
:- att(X,Y), in(X), in(Y).§\label{line:cf-3}§
\end{lstlisting}

\begin{lstlisting}[caption=Module $\encdefense$,label=asp:def,float=tb,belowskip=-0.8 \baselineskip]
defeated(X) :- in(Y), att(Y,X).§\label{line:adm-r4}§
undefended(X) :- att(Y,X), not defeated(Y).§\label{line:adm-r5}§
:- in(X), undefended(X).§\label{line:adm-r6}§
\end{lstlisting}

Correctness of the encodings $\encf$ and $\enadm$ was proven in~\cite{EglyGW2010}.
\begin{prop} \label{th:cf} \label{th:adm}
For any AF $\AF$, we have 
(i) $\cf(F) \cong \mathcal{AS}(\encf(\enf))$, and 
(ii) $\adm(F) \cong \mathcal{AS}(\enadm(\enf))$.
\end{prop}

Next, we characterize the encoding $\enrange$ (\Cref{asp:rng}), which, given a module computing some extension $S$ (via $\inA$) of an AF $(A,R)$, returns its range 
$S^{+}_{R}$ 
(via $\rge$) and also collects
the arguments not contained in the range. We indicate via 
$\ok$ that $S$ is not stable, i.e.\ 
$S^{+}_{R}\subset A$.

\begin{lstlisting}[caption=Module $\enrange$,frame=lines,label=asp:rng,escapechar=§,float=tb,belowskip=-0.8 \baselineskip]
range(X) :- in(X).§\label{line:rng-r1}§
range(Y) :- in(X),att(X,Y).§\label{line:rng-r2}§
out_of_range(X) :-  not range(X),arg(X).§\label{line:rng-r3}§
unstable :-  out_of_range(X),arg(X).§\label{line:rng-r4}§
\end{lstlisting}

\begin{restatable}{lma}{proprange}
\label{prop:5}
Let $\AF=(A,R)$ be an AF, and 
$\aspmodule{}$ be a program not containing the predicates $\rge(\cdot)$, $\nrge(\cdot)$ and $\ok$. 
Let $I \subseteq B_A$ and $S \subseteq A$ s.t. $I \cong S$.
Furthermore let $\aspmodule{}^{+} = \aspmodule{} \cup \enrange$ and
\begin{align}
\label{eq:range-condition}
\begin{split}
I^{+} = I &\cup \{\rge(a) \mid a \in S^{+}_R\} 
\cup \{\nrge(a) \mid a \in A \setminus S^{+}_R\} 
\\ &\cup \{\ok \mid S^{+}_R \subset A\}.
\end{split}
\end{align}
Then, $I \in \AS (\aspmodule{}(\enf))$, if and only if 
$I^{+} \in \AS (\aspmodule{}^{+}(\enf))$.
\end{restatable}

\begin{lstlisting}[caption=Module $\aspmodule{eq}$,frame=lines,label=asp:eq,escapechar=§,float=tb,belowskip=-0.8 \baselineskip]
eq_upto(Y) :- inf(Y), in(Y), inN(Y).
eq_upto(Y) :- inf(Y), out(Y), outN(Y).
eq_upto(Y) :- succ(Z,Y), in(Y), inN(Y), eq_upto(Z).
eq_upto(Y) :- succ(Z,Y), out(Y), outN(Y), eq_upto(Z).
eq :- sup(Y), eq_upto(Y). 
\end{lstlisting}

The preferred, semi-stable~\cite{EglyGW2010} and stage semantics~\cite{DvorakGWW2013} utilize the so-called \textit{saturation technique}. We sketch here the basic ideas. Intuitively, in the saturation technique  encoding for preferred semantics we make a first guess for a set of arguments in the framework, and then we verify if this set is admissible (via module $\aspmodule{adm}$). 
To verify if this set is also subset maximal admissible, a second guess is carried out via a disjunctive rule. If this second guess corresponds to an admissible set that is a proper superset of the first one, then the first one cannot be a preferred extension. 
Using the saturation technique now ensures that if all second guesses ``fail'' to be a strictly larger admissible set of the first guess, then there is one answer-set corresponding to this preferred extension.  
Usage of default negation within the saturation technique for the second guess is restricted, and thus a loop-style encoding is employed that checks if the second guess is admissible and a proper superset of the first guess. 

Roughly, a loop construct in ASP checks a certain property for the least element in a set (here we use the predicate $\aspinf(\cdot)$), and then checks this property ``iteratively'' for each (immediate) successor (via predicate $\aspsucc(\cdot,\cdot)$). If the property holds for the greatest element ($\aspsup(\cdot)$), it holds for all elements. In \Cref{asp:eq} we illustrate loop encodings, where we see a partial ASP encoding used for preferred semantics in~\cite{EglyGW2010} that derives \aspformat{eq} 
if the first and second guesses are equal, i.e.\ the predicates corresponding to the guesses via $\inA(\cdot)$, resp.\ $\outA(\cdot)$, and $\inn(\cdot)$, resp. $\outn(\cdot)$, are true for the same constants. 

Another variant of ASP encodings for preferred, semi-stable and stage semantics is developed by~\cite{DvorakGWW2013}. There so-called meta-asp encodings are used, which allow for minimizing statements w.r.t.\ subset inclusion directly in the ASP language~\cite{GebserKS11}.
For instance, $\aspmodule{adm}$ can then be augmented with a minimizing statement on the predicate \aspformat{out}, to achieve an encoding of preferred semantics.

\section{Encodings}
\label{sec:encodings}

Here we present our new encodings for preferred, semi-stable, and stage semantics via the novel characterizations. 

\subsection{Encoding for Preferred Semantics}

The encoding for preferred semantics is given by $\enpreftwo = \enadm \cup \ensatpreftwo$, where $\ensatpreftwo$ is provided in~\Cref{asp:pref-alt}. 
We first give the intuition of the program. 
A candidate $S$ for being preferred in an AF $F=(A,R)$ is computed by the program $\enadm$
via the $\inA(\cdot)$ predicate, and is already known admissible. If all arguments
in $A$ are contained in $S$ we are done\footnote{Note, this is only the case when there are no attacks in $F$.}. 
 Otherwise, the remainder of the program
$\ensatpreftwo$
(\Cref{line:pr-r02,line:pr-r7}) is used to check whether there exists a set 
$E\in\adm(F)$ such that $E\not\subseteq S$ and not in conflict with $S$. 
We start to build $E$ by 
guessing some argument not contained in $S$ (\Cref{line:pr-r02}) and then in \Cref{line:pr-r03} we repeatedly add further arguments to $E$ unless the set defends itself (otherwise we eventually derive
$\spoil$). 
Then, we check whether $E$ is conflict-free (\Cref{line:pr-r4}) and $E$ is not in conflict with $S$ (\Cref{line:pr-r5}). 
If we are able to reach this point without deriving $\spoil$, then the candidate $S$ cannot be an answer-set
 (\Cref{line:pr-r7}). This is in line with 
Proposition~\ref{prop:2}, which states that in this case $S$ is not preferred.

By inspecting~\Cref{asp:pref-alt} we also see important differences w.r.t.\ the encodings for preferred semantics of~\cite{EglyGW2010}. In our new encodings, the ``second guess'' via predicate $\ecl(\cdot)$ is constructed through conditional disjunction instead of simple disjunction. Usage of the former allows to construct the witness set already with defense of arguments in mind. 
Furthermore loops, such as the one shown in~\Cref{asp:eq} that checks if the second guess is equal to first one or a loop construct that checks if every argument is defended, can be avoided, since these checks are partially incorporated into~\Cref{line:pr-r02} of~\Cref{asp:pref-alt} and into simpler further checks. 

\begin{lstlisting}[float=tb,caption=Module $\ensatpreftwo$,frame=lines,label=asp:pref-alt,escapechar=§,belowskip=-0.8 \baselineskip]
nontrivial :- out(X).§\label{line:pr-r01}§
witness(X):out(X) :- nontrivial.§\label{line:pr-r02}§
spoil | witness(Z):att(Z,Y) :- witness(X), att(Y,X).§\label{line:pr-r03}§
spoil :- witness(X), witness(Y), att(X,Y).§\label{line:pr-r4}§
spoil :- in(X), witness(Y), att(X,Y).§\label{line:pr-r5}§
witness(X) :- spoil, arg(X).§\label{line:pr-r6}§
:-  not spoil, nontrivial.§\label{line:pr-r7}§
\end{lstlisting}

Correctness of this new encoding is stated and proved in the following proposition. 

\begin{prop}
\label{prop:pref}
For any AF $\AF$, we have $\pref(\AF) \cong \AS (\enpreftwo (\enf))$.
\end{prop}
\begin{proof*}
According to~\Cref{def:correspondence}, we have to prove (i) and (ii). With line numbers we refer here to the ASP encoding shown in \Cref{asp:pref-alt}. We employ the \textit{splitting theorem}~\cite{LifschitzT94} in order to get a characterisation of $\AS(\enpreftwo(\enf ))$, in which the sub-programs $\ensatpreftwo$ and $\enadm$ are considered separately. The splitting set is $C_{\pref ^2}$ $=$ $\{\argument(\cdot)$, $\att(\cdot,\cdot)$, $\inA(\cdot)$, $\outA(\cdot)$, $\defeated(\cdot)$, $\undefended(\cdot)\}$, and we obtain
\begin{equation} \label{eq:split}
\AS (\enpreftwo (\enf )) = \bigcup_{J \in \AS (\enadm (\enf ))} \AS (J \cup \ensatpreftwo ). 
\end{equation}
\paragraph{Proof (i).} We prove that each preferred extension $S \in \pref(\AF)$ has a corresponding answer-set $I \in \AS ( \enpreftwo (\enf))$. 
From~\Cref{eq:split} we know that $I \in \AS ( \enpreftwo (\enf))$ if $I \in \AS (J \cup \ensatpreftwo )$, for some $J \in \AS ( \enadm (\enf))$. 
Moreover $S \in \pref(\AF)$ implies $S \in \adm(\AF)$, hence by \Cref{th:cf} there is $J \in \AS (\enadm (\enf))$ s.t. $J \cong S$. In the following we distinguish between two complementary cases.

In  case $R = \emptyset$, the set $S=A$ is the only preferred one, since it is trivially admissible and it  cannot be contained in another set of arguments.  
We show $I=J$ is a subset-minimal model of $(J \cup \ensatpreftwo)^J$. 
The subset-minimality is evident. 
Then, $\outA(a) \notin J$ for any $a \in A$ by \Cref{prop:partition}, hence $J$ satisfies the rule at \Cref{line:pr-r01}.
Since $\nontrivial \notin J$, $J$ satisfies the rules at Lines \ref{line:pr-r02}, and \ref{line:pr-r7}. 
Every other rule is satisfied because $\att(a,b) \notin \hat{F}$ for any $a,b \in A$.

In  case $R \neq \emptyset$ 
%
%
%
we can build an interpretation $I$ and prove that $I$ is an answer-set by contraposition, i.e. if there is an $L \subset I$
 which satisfies $(J \cup \ensatpreftwo)^I$, then $S \notin \pref(\AF)$.
We define
$
I = J \cup \{\spoil, \nontrivial\} \cup \{\ecl(a) \mid a \in A\}
$.
We have $I \cong S$ since $I \equiv J$.
The set $I$ satisfies $(J \cup \ensatpreftwo)^I$ (got from $\Gr(J \cup \ensatpreftwo)$ by just removing the rule at~\Cref{line:pr-r7}), 
as $J \subseteq I$ and $I$ contains all the heads of the rules in $(J \cup \ensatpreftwo)^I$. 
Notice that $R\neq \emptyset$ guarantees that the head of the rule at~\Cref{line:pr-r02} is non-empty.

Now we describe the necessary shape of $L$, in order to prove the main assertion next. 
$L$ must contain $\nontrivial$ because of the rule at~\Cref{line:pr-r01}. Indeed $\outA(c) \in L$ for some $c \in A \setminus S$, since $J \subseteq L$ 
with $J\cong S$ and $S \in \cf(\AF)$ (since $S \in \adm(\AF)$), which implies the existence of $c \in A \setminus S$ (we cannot have simultaneously $R \neq \emptyset$, $S \in \cf(\AF)$ and $S=A$), which implies $\outA(c) \in J$ by~\Cref{prop:partition}. 
We have $\spoil \notin L$, otherwise also $\{\ecl(a) \mid a \in A\}$ would be in $L$ (because of the rule at~\Cref{line:pr-r6}), making $L$ equal to $I$, but they are different by assumption. 
%


Now we show that, given $L$, it is possible to find a set $U \in \adm(\AF)$ s.t. $U \not \subseteq S$ and $U \cup S \in \cf(\AF) $, which implies $S \notin \pref(\AF)$ by~\Cref{prop:2}. We define $U = \{a \mid \ecl(a) \in L \}$, and we show all the required properties:

%
\noindent
\underline{$U\in\cf(F)$}, otherwise we would have two arguments $a,b$ attacking each other, meaning $\{\ecl(a),$ $\ecl(b), \att(a,b)\} \subseteq L$, which implies $B(r) \subseteq L$  and $H(r) \not \subseteq L$ for some rule $r$ in the grounding of the rule at~\Cref{line:pr-r03}, since $\spoil \notin L$. 
%

\noindent
\underline{Each $a\in U$ is defended by $U$}, otherwise it would be possible to find two atoms $\ecl(a) \in L$ [$a \in U$]\footnote{\label{foot:squarebrackets} In this proof, the square brackets are used to point out an immediate implication of the statement preceding them. Usually the statement is about the framework $\AF$ and the implication about an interpretation, or the other way around.} 
and $\att(b,a) \in L$ [$( b,a) \in R$] for which there is no $\ecl(c) \in L$ [$c \in U$] s.t. $\att(c,b) \in L$ [$( c,b) \in R$], thus violating 
the rule at~\Cref{line:pr-r4}, since $\spoil \notin L$.
%

\noindent
\underline{$U\not\subseteq S$}. Indeed if we assume $U \subseteq S$, then for every $\ecl(a) \in L$ we have $a \in S$ (by definition of $U$), which corresponds to $\inA(a) \in J$ ($S \cong J$), implying $\outA(a) \notin J$ (by~\Cref{prop:partition}), making it impossible for $L$ to satisfy the rule at~\Cref{line:pr-r02}, since $\nontrivial \in L$.

\noindent
\underline{$\{U \cup S\}\in \cf(F)$}. The sets $U$ and $S$ are conflict-free, so we have to show that there cannot be attack relations between the two sets:
an argument $a \in S$ cannot attack an argument $b \in U$, otherwise we would have $\{ \ecl(b)$, $\inA(a)$, $\att(a,b) \}$ $\subseteq$ $L$, which implies $B(r) \subseteq L$  and $H(r) \not \subseteq L$ for some rule in the grounding of the rule at~\Cref{line:pr-r5}, since $\spoil \notin L$;
an argument $b \in U$ cannot attack an argument $a \in S$, otherwise an argument $c \in S$ should attack $b$ by admissibility of $S$, thus violating the previous point.
\vspace{-0.3cm}
\paragraph{Proof (ii).}
We prove that each $I \in \AS (\enpreftwo(\enf))$ corresponds to an $S \in \pref(\AF)$. 
From~\Cref{eq:split} we see that $I \in \AS (\enpreftwo (\enf))$ only if $I \in \AS (J \cup \ensatpreftwo)$ for some $J \in \AS (\enadm (\enf ))$. We have $I \equiv J$, because $J \subseteq I$, and $I$ does not have any additional ground atom $\inA(a)$, since $\inA(\cdot)$ does not appear in the head of any rule of $\ensatpreftwo$.
By~\Cref{th:adm} there exists $S \in \adm(\AF)$ s.t. $S \cong J$, hence $S \cong I$ by~\Cref{th:equiv}. 
We show that $S$ is also preferred in $\AF$, by distinguishing between two complementary cases.

\noindent
\underline{$\nontrivial \notin I$}: we have $\outA(a) \notin I$ for any $a \in A$, otherwise the rule at \Cref{line:pr-r01} would be violated. By Proposition\ref{prop:partition} this implies $\inA(a) \in I$ for every $a \in A$, and the same is true for $J$ ($J \cong I$), which we know to be admissible. 
Hence, $S=A$ and $S\in \pref(\AF)$.

\noindent
\underline{$\nontrivial \in I$}: we prove that $S$ is preferred by contraposition, i.e. if $S \notin \pref(\AF)$ then $I$ is not a subset-minimal model of $(J \cup \ensatpreftwo)^I$.
We have that $I$ must have a clear shape in order to satisfy $(J \cup \ensatpreftwo)^I$. In particular $J \subseteq I$.
Then $\spoil \in I$ because of the rule at~\Cref{line:pr-r7} 
hence, $\ecl(a) \in I$ for each $\argument(a) \in I$ because of the rule at~\Cref{line:pr-r6}. 
Summing up we have $J \cup \{\nontrivial, \spoil\} \cup \{\ecl(a) \mid a \in A\} \subseteq I$.
Finally we show that $I \notin \AS (J \cup \ensatpreftwo)$, since we are able to build an interpretation $L \subset I$ satisfying the reduct 
$(J \cup \ensatpreftwo )^I$.
We remind that $S \notin \pref(\AF)$ means that there exists $T \in \pref(\AF)$ s.t. $S \subset T$.
We use $T$ to build the interpretation $L = J \cup \{\nontrivial\} \cup \{ \ecl(a) \mid a \in T \}$. 
We have $L \subset I$, because it does not contain $\spoil$ and $T \subseteq A$.
In the following we show that $L$ is a model of the reduct, because it contains $J$ and it satisfies each rule in $\Gr(\ensatpreftwo)$.
	
	$L$ satisfies the rule at~\Cref{line:pr-r02}, because there exists $\ecl(a) \in L$ s.t. $\outA(a) \in L$, for some $a \in T \setminus S $ (the element $a$ exists because $T$ is a proper superset of $S$). 
	\footnote{\label{foot:3} If $a \in T \setminus S $, then $a \notin S$, then $\inA(a) \notin I$ ($S \cong I$), then $\inA(a) \notin J$ ($I \equiv J$), then $\outA(a) \in J$ (by~\Cref{prop:partition}), then $\outA(a) \in L$ ($J \subset L$). Summing up, if $a \in T \setminus S$, then $\outA(a) \in L$, and $\ecl(a) \in L$ by definition.}
	
	 Since $T$ is admissible, for each $a \in T$ [$\ecl(a) \in L$] 
	 attacked by  $b \in A$ [$\att(b,a) \in \enf $] there exists $c \in T$ [$\ecl(c) \in L$] attacking $b$ [$\att(c,b) \in \enf $]. Hence $L$ satisfies  the rule at~\Cref{line:pr-r03}, even though $\spoil \notin L$.
	
	$L$ does not contain the body of any rule in the grounding of the rule at~\Cref{line:pr-r4}, otherwise $T$ would not be conflict free.
	$L$ does not contain the body of any rule in the grounding of the rule at~\Cref{line:pr-r5}, otherwise $T$ would not be conflict free, since $S \subset T$.
	$L$ does not contain the body of any rule in the grounding of the rule at~\Cref{line:pr-r6}, because it does not contain $\spoil$. \phantom{.}\hfill\ensuremath{\Box}
\end{proof*}



%

\subsection{Encodings for Semi-Stable and Stage Semantics}


\paragraph{Semi-stable semantics}
The encoding for semi-stable semantics is given by $\ensemitwo = \enadm \cup \enrange \cup \ensatsemitwo$, with  $\ensatsemitwo$ shown in \Cref{asp:semi-alt}.
We first give the intuition.
A candidate $S$ for being semi-stable is computed by the program $\enadmp = \enadm \cup \enrange$ via the $\inA(\cdot)$ predicate and is known admissible. The module $\enrange$ computes the range and derives $\ok$ iff the extension is not stable. If $S$ is stable, we are done. Otherwise the remainder of the program $\ensatsemitwo$ is used to check whether an admissible cover $E$ of a superset of the range $S_R^{+}$ exists. 
Starting from $S_R^{+}$ (\Cref{line:sm-r2}), a superset is achieved by adding at least one element out of it (\Cref{line:sm-r1}). Then a cover is found (\Cref{line:sm-r3}), which is admissible (\Cref{line:sm-r3,line:sm-r4}). If we are able to reach this point without deriving $\spoil$ (that is always a possibility for satisfying the constraints), then the candidate $S$ cannot be an answer-set (\Cref{line:sm-r3}). This is in line with \Cref{prop:1}, which states that in this case $S$ is not semi-stable. 
Here we state the correctness of the encoding, a full proof is given 
in the online appendix (Appendix A). 

\begin{lstlisting}[float=tb,caption=Module $\ensatsemitwo$,frame=lines,label=asp:semi-alt,escapechar=§,belowskip=-0.8 \baselineskip]
larger_range(X):out_of_range(X) :- unstable.§\label{line:sm-r1}§
larger_range(X) :- range(X), unstable.§\label{line:sm-r2}§
witness(X) | witness(Z):att(Z,X) :- larger_range(X), unstable.§\label{line:sm-r3}§
spoil :- witness(X), witness(Y), att(X,Y), unstable.§\label{line:sm-r4}§
spoil | witness(Z):att(Z,Y) :- witness(X), att(Y,X), unstable.§\label{line:sm-r5}§
witness(X) :- spoil, arg(X), unstable.§\label{line:sm-r6}§
larger_range(X) :- spoil, arg(X), unstable.§\label{line:sm-r7}§
:- not spoil, unstable.§\label{line:sm-r8}§
\end{lstlisting}

\begin{restatable}{prop}{propSemiCorrectness}
\label{prop:semi}
 For any AF $\AF =(A,R)$, we have $\semi(F) \cong \AS (\ensemitwo (\enf))$.
\end{restatable}

\paragraph{Stage semantics}


The encoding for stage semantics is given by $\enstagetwo = \encf \cup \enrange \cup \ensatsemitwo \setminus \{r_{admcov}\}$, where $r_{admcov}$ is the rule at \Cref{line:sm-r5} of \Cref{asp:semi-alt}. 
The only differences w.r.t.\ the encoding for semi-stable semantics are: (i) it employs $\encf$ instead of $\enadm$, thus the candidate sets are only conflict-free; and (ii) it lacks the rule at \Cref{line:sm-r5}, hence it considers all the conflict-free covers of the candidate set, which is still in line with \Cref{prop:1}. 
A proof sketch for the forthcoming correctness result is given 
in the online appendix (Appendix A). 


\begin{restatable}{prop}{propStageCorrectness}
For any AF $\AF =(A,R)$, we have $\stage(\AF) \cong \AS (\enstagetwo (\enf))$.
\end{restatable}

\section{Evaluation}

We tested the novel encodings (NEW) extensively and compared them to the original (ORIGINAL) and metasp (META) encodings as well as to the system ConArg~\cite{BistarelliS11a}. For the novel and original encodings we used
\emph{Clingo 4.4} and for the metasp encodings we used \emph{gringo3.0.5/clasp3.1.1} all from the Potassco 
group\footnote{\url{http://potassco.sourceforge.net}}.
As benchmarks, we considered a collection of frameworks which have been used by different colleagues for testing before consisting of structured and
randomly generated AFs, resulting in 4972 
frameworks. In particular we used parts of the instances Federico Cerutti provided to us which have been generated towards
an increasing number of SCCs~\cite{VallatiCG14}. 
Further benchmarks were used to test the system 
\emph{dynpartix} and we included the instances provided by the ICCMA 2015 organizers.
The full set is available at \url{http://dbai.tuwien.ac.at/research/project/argumentation/systempage/#conditional}.

For each framework the task is to enumerate all solutions. 
The computation has been performed on an Intel Xeon E5-2670 running at 2.6\,GHz. 
From the 16 available cores we used only every fourth core to allow a better utilization of the CPU's cache. 
We applied a 10 minutes timeout, allowing to use at most 6.5\,GB of main memory. 

It turns out that for each semantics the new encodings significantly outperform the original ones as well as the system ConArg. 
Furthermore, there is a clear 
improvement to the metasp encodings, as illustrated in 
Fig.~\ref{fig:runtimes} which shows the cactus plots of the required runtime to solve frameworks (x-axis) with the respective timeout (y-axis) for the three
 discussed semantics. 
While for preferred and semi-stable semantics the novel encodings are able to solve more than 4700 instances (out of 4972), one can 
observe a different trend for stage semantics. There, the new encodings return the best result with 2501 solved instances.
\begin{figure}[t]
\centering
\includegraphics[width=6.2cm]{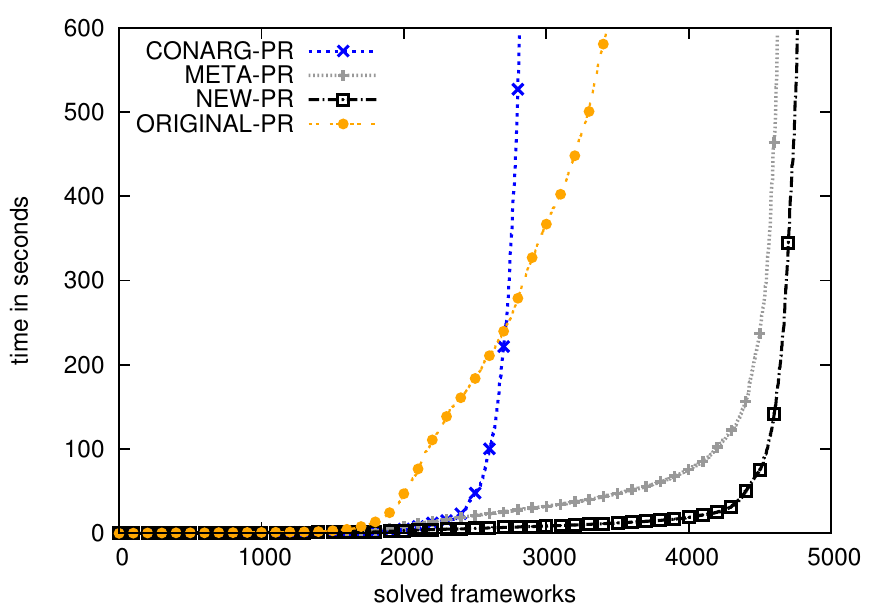} \includegraphics[width=6.2cm]{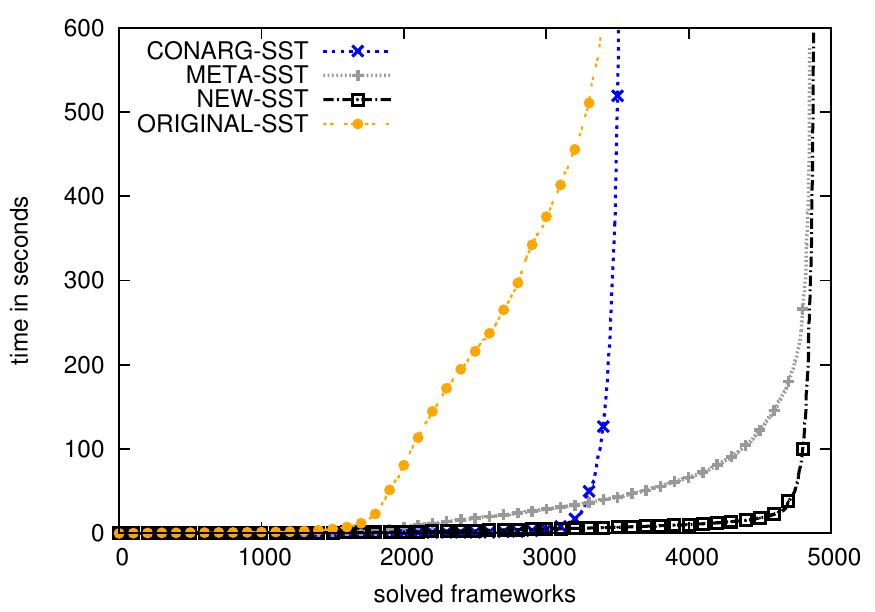}
\vspace{0.5cm}

\includegraphics[width=6.2cm]{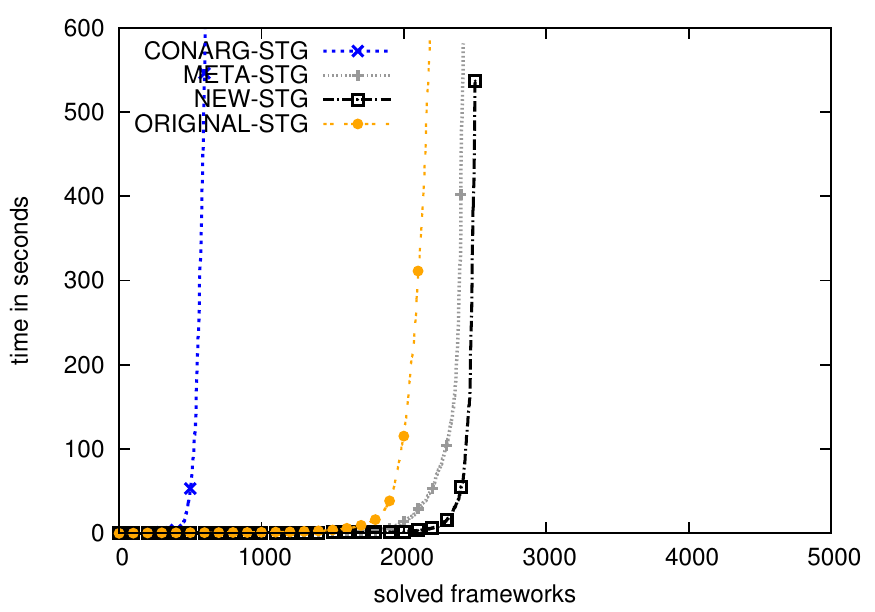}
\caption{Runtimes for preferred (PR), semi-stable (SST) and stage (STG) semantics.}
\label{fig:runtimes}
\end{figure}
Table~\ref{tab:summary} gives a summary of the test results, where \emph{usc} denotes the unique solver contribution, i.e.\ 
the number of AFs which could only be solved by the particular solver, \emph{solved} gives the number of solved instances
by the solver, and \emph{med} is the median of the computation time of the solver.
\begin{table}[t]
\begin{center}
\begin{tabular}{p{1.1cm}|p{0.2cm}p{0.4cm}r||p{1.1cm}|p{0.2cm}p{0.4cm}r||p{1.1cm}|p{0.2cm}p{0.4cm}r}
PR & usc & solved &  med & SST & usc & solved &  med & STG & usc & solved &  med\\
\hline
ConArg & 60 & 2814 &  43.65 & ConArg & 50 & 3509 & 1.03 & ConArg & - & 606 & 600.00\\
Original & - & 3425 &  180.36 & Original & - & 3386 & 211.96 & Original & - & 2185 & 600.00\\
Meta & 1 & 4626 &  20.83 & Meta & 2 & 4830 & 17.70 & Meta & 5 & 2419 & 600.00 \\
New & 101 & 4765 &  5.77 & New & 13 & 4879 & 3.30 & New & 82 & 2501 & 384.92\\
\end{tabular}
\caption{Summary of test results.\label{tab:summary}}
\end{center}
\end{table}
Interestingly, ConArg is able to solve 60 (resp.\ 50) instances for preferred (resp.\ semi-stable) semantics which are not
solvable by the other systems. However, the novel encodings are able to uniquely solve 101 (resp.\ 82) instances for preferred (resp.\ stage)
semantics. The original encodings have no \emph{unique solver contribution} for all of the considered semantics, thus it is save to replace them with
the new encodings. The entries for the median also show that all the novel encodings perform much faster than the other
systems, except for semi-stable where ConArg has the lowest median. 
However, here ConArg is able to solve about 1300 instances
less than the novel encodings.

Another interesting observation is that the grounding size of all new encodings is significantly smaller than of both the original and 
the metasp encodings. 

\section{Conclusion}


In this work, we have developed novel ASP encodings for computationally challenging problems arising in abstract argumentation. 
%
%
Our new encodings for preferred, semi-stable, and stage semantics avoid complicated loop constructs present in previous encodings. In addition to being more succinct, 
our empirical evaluation showed that a significant performance boost was achieved compared to the earlier ASP encodings, and that our encodings outperform the 
%
%
state-of-the-art system
ConArg.  From an ASP perspective, our results indicate that 
loops 
in saturation encodings 
(as used in the previous encodings in \cite{EglyGW2010}) 
are a severe performance bottleneck which should be avoided.

In future work, we plan to compare our results also with 
the systems 
CEGARTIX
\cite{DvorakJWW2014}
and  ArgSemSAT
\cite{CeruttiGV14a}.
Furthermore, we also aim for finding better ASP encodings
for the ideal~\cite{DungMT07} and eager semantics~\cite{Caminada2007}.

\end{document}